\documentclass[11pt]{article}

\usepackage[
  shownumpages,
  bgcolor={250,248,241},
  braincolor={133,76,101},
  linkcolor={133,76,101},
  citecolor={133,76,101},
  urlcolor={133,76,101},
  citingstyle=authoryear,
  bibliostyle=unsrtnat,
  bibfile=refs
]{Styles/brainlab}

\usepackage{microtype}
\usepackage{mathtools}
\usepackage{bm}
\usepackage{multirow}
\usepackage{tabularx}
\usepackage{longtable}
\usepackage{threeparttable}
\usepackage{subcaption}
\usepackage{enumitem}
\usepackage{wrapfig}

\definecolor{HWPrimary}{HTML}{854C65}
\definecolor{HWSecondary}{HTML}{F0C987}
\definecolor{HWSoft}{HTML}{FBEDD6}
\definecolor{HWBg}{HTML}{FAF8F1}
\definecolor{HWAccent}{HTML}{854C65}

\usepackage{tikz}
\usetikzlibrary{calc,positioning,arrows.meta,shapes,fit,decorations.pathreplacing}
\usepackage{pgfplots}
\pgfplotsset{compat=1.18}
\usepgfplotslibrary{fillbetween}

\graphicspath{{plots/}{./plots/}}

\setlist[itemize]{leftmargin=1.2em,itemsep=0.35em,topsep=0.35em}
\setlist[enumerate]{leftmargin=1.4em,itemsep=0.35em,topsep=0.35em}

\DeclareBrainTcbTheorem{lemma}{Lemma}
\DeclareBrainTcbTheorem{proposition}{Proposition}
\DeclareBrainTcbTheorem{remark}{Remark}
\DeclareBrainTcbTheorem{example}{Example}
\DeclareBrainTcbTheorem{note}{Note}
\newtheorem*{theorem*}{Theorem}

\DeclareMathOperator*{\argmin}{arg\,min}

\setbrainmeta{
  title={Spectral Weight Decay: Inducing Low-Rank Structure in Neural Network Weights},
  authors={
    Dmitrii Andriianov\textsuperscript{1},
    Andrey Veprikov\textsuperscript{1,2},
    Aleksandr Beznosikov\textsuperscript{1,3}
  },
  affiliations={
    \textsuperscript{1}Basic Research of Artificial Intelligence Laboratory (BRAIn Lab)\\
    \textsuperscript{2}SB AI Lab\\
    \textsuperscript{3}Innopolis University
  },
  abstract={
    Standard weight decay treats each weight matrix as a vector and ignores its
    spectral structure. We introduce spectral weight decay, a post-step
    decoupled nuclear-norm update that applies additive rather than
    multiplicative spectral shrinkage. We connect the update to approximate
    proximal descent and show that its sensitivity to update order can exceed
    that of conventional $\ell_2$ weight decay near rank deficiency.
    Across LLaMA models with $124$M to $500$M parameters, spectral
    weight decay lowers effective rank and improves SVD-LLM compression at
    matched validation loss. At $500$M and a $4\%$ distortion budget, it reaches
    $1.89\times$ compression and $1.18\times$ GPU inference speedup, compared
    with $1.14\times$ and $1.01\times$ after standard weight decay. Under
    fixed-horizon training with $60\%$ label noise, it also improves final
    mean clean-test accuracy over matched $\ell_2$ regularization by up to $17.8$
    points on MNIST and $4.6$ points across four BERT-base tasks.
    Code is available at \url{https://github.com/brain-lab-research/SpectralWD}.
  },
}

\begin{document}
\begin{mainpart}

\section{Introduction}
\label{sec:intro}

Optimization is central to deep learning because it governs training stability,
efficiency, and the properties of the learned solution. AdamW has become the
standard optimizer in modern training recipes~\citep{loshchilov2019decoupled}.
At the same time, matrix-aware optimizers such as Shampoo, SOAP, and Muon have
begun to exploit the geometry of neural-network weight matrices instead of
treating each layer as a flat parameter
vector~\citep{gupta2018shampoo,vyas2024soap,jordan2024muon}. Although their
update rules differ, all retain conventional weight decay as a core component.

Weight decay remains part of these training recipes because it improves
generalization~\citep{loshchilov2019decoupled}, preserves downstream
plasticity~\citep{han2026weight}, and stabilizes
training~\citep{dangelo2024why}. AdamW decouples weight shrinkage from
adaptive preconditioning, so coordinatewise scaling of the loss gradient does
not also reshape the decay term. The regularizer itself remains matrix-agnostic:
its squared Frobenius penalty depends on the sum of squared singular values
without explicitly favoring the removal of small ones.

To make the regularizer matrix-aware, we begin with the proximal interpretation
of decoupled weight decay. To first order, its update is a proximal gradient
step for the squared Frobenius
penalty~\citep{parikh2014proximal,zhuang2022understanding}. We replace this penalty with the nuclear
norm, a convex surrogate for rank whose proximal operator soft-thresholds
singular values~\citep{recht2010guaranteed,cai2010singular}. This construction
is the matrix analogue of replacing ridge shrinkage with coefficient selection
by the lasso~\citep{hoerl1970ridge,tibshirani1996regression}. Given the
loss-only optimizer direction $\Delta$, let
$Z=W^{(k)}-\eta\Delta=U_Z\Sigma_ZV_Z^\top$. Spectral weight decay applies the
post-step correction
\begin{equation}
  W^{(k+1)}=Z-\eta\lambda U_ZV_Z^\top,
  \label{eq:swd_decoupled}
\end{equation}
where $\eta$ is the learning rate and $\lambda$ the regularization coefficient.
This correction acts additively in the singular basis of $Z$, whereas
conventional weight decay scales the spectrum uniformly. 

SLORR-Nuc also
approximates the nuclear-norm polar
factor~\citep{gonzalezmartinez2026slorr}. Its decoupled extension would use
the pre-step correction $Z-\eta\lambda U_WV_W^\top$ for
$W^{(k)}=U_W\Sigma_WV_W^\top$, rather than \eqref{eq:swd_decoupled}.

The additive shrinkage acts more strongly on smaller singular values and
concentrates the spectrum. The resulting low-effective-rank weights admit
efficient truncated SVD compression. Replacing a dense $m\times n$ matrix with
rank-$r$ factors reduces its parameter count and multiplication cost from $mn$
to $r(m+n)$. This factorization lowers storage and inference memory and can
reduce latency when the retained rank is sufficiently small. Beyond
compression, we test whether the same low-rank bias limits memorization under
label noise.

We evaluate spectral weight decay in LLaMA~\citep{touvron2023llama} pretraining,
post-training compression, and fixed-horizon noisy-label classification. Across
models with $124$M to $500$M parameters, it lowers effective
rank~\citep{roy2007effective} and improves compression and inference speed at
matched validation loss. At $500$M and a $4\%$ distortion budget, it reaches
$1.89\times$ compression and a $1.18\times$ GPU inference speedup. At $60\%$
label noise, mean clean-test accuracy gains over matched $\ell_2$ regularization
reach $17.78$ points on MNIST and $4.59$ points on four BERT-base tasks.

\begin{samepage}
\paragraph{Contributions}

\begin{itemize}
  \item We formulate spectral weight decay as decoupled nuclear-norm
  regularization, derive its additive singular-value shrinkage, and connect the
  update to approximate proximal descent.
  \item Across LLaMA models with $124$M to $500$M parameters, spectral weight
  decay improves low-rank compression and inference speed over standard weight
  decay at matched validation loss. It also outperforms Cuttlefish and Prehab
  in matched compression comparisons.
  \item On image and text classification tasks, spectral weight decay improves fixed-horizon
  robustness to label noise over matched $\ell_2$ regularization.
\end{itemize}
\end{samepage}


\section{Related Work}
\label{sec:rel_work}

\paragraph{Weight decay.}
Early work explains weight decay as a regularizer that improves generalization
by suppressing irrelevant parameter directions and sensitivity to target
noise~\citep{krogh1991simple}. For SGD, weight decay is equivalent to an
$\ell_2$ penalty up to rescaling the coefficient. This equivalence breaks
under adaptive preconditioning, motivating the decoupled update used by
AdamW~\citep{loshchilov2019decoupled}. Analyses show that its effect
depends on the optimizer, architecture, learning rate, and training horizon.
The benefit can arise through changes in effective learning rate and
optimization dynamics rather than norm control
alone~\citep{zhang2019three,lewkowycz2020training,dangelo2024why}. Weight decay
interacts with matrix structure. In factorized attention layers,
$\ell_2$ regularization of the factors is linked to nuclear-norm regularization
of their product and can induce low-rank key-query and value-projection
maps~\citep{kobayashi2024weight}. This is an instance of a general equivalence:
minimizing the $\ell_2$ penalties of two factors over all factorizations
$W=UV^\top$ is equivalent to penalizing the nuclear norm of
$W$~\citep{srebro2004maximum}. As a convex rank
surrogate~\citep{recht2010guaranteed}, the nuclear norm can instead regularize
each weight matrix directly.

\paragraph{Nuclear-norm regularization.}
Direct nuclear-norm penalties act on the learned weight matrices without fixing
a rank in advance, but exact gradients and proximal updates require spectral
decompositions. \citet{alvarez2017compression} use periodic singular-value
thresholding to prepare networks for post-training SVD compression. SLORR
regularizes dense weight matrices using approximations of the polar
factor~\citep{gonzalezmartinez2026slorr}. Its variants use either the Hoyer
penalty or the nuclear norm. SLORR-Nuc applies
the penalty through the loss. The paper also describes a decoupled pre-step
rule, but evaluates it only for SLORR-Hoyer, not SLORR-Nuc. This rule computes
the polar factor at $W$, while our post-step update \eqref{eq:swd_decoupled}
uses that of $Z$. SLORR develops efficient low-rank regularization. We replace
the squared Frobenius penalty of standard weight decay with the nuclear norm
and analyze the post-step update as approximate proximal descent.
NuMuon instead acts on the optimizer update rather than the weights. It
constrains the nuclear norm of Muon's update direction to promote low-rank
updates~\citep{dolatabadi2026numuon}.

\paragraph{Low-rank training.}
Another line of work imposes explicit low-rank structure during training.
Cuttlefish begins with dense training, waits for the stable ranks of the layers
to converge, and then switches to a factorized model at the corresponding
ranks~\citep{wang2023cuttlefish}. TRP periodically replaces weight matrices
with truncated-SVD approximations, selects the retained rank using a
spectral-energy criterion, and continues gradient updates in the original
parameterization~\citep{xu2020trp}. These projections can be combined with
nuclear-norm regularization. Spectral weight decay instead retains dense
weight matrices throughout training and does not select a hard rank until
compression. Appendix~\ref{app:alternative_methods} compares spectral weight
decay with Cuttlefish.

\paragraph{Post-training low-rank compression.}
SliceGPT removes rows and columns after a principal-component
transformation~\citep{ashkboos2024slicegpt}.
ASVD~\citep{yuan2023asvd} and SVD-LLM~\citep{wang2024svdllm} use activation
statistics to improve the factorization. Dobi-SVD learns layerwise truncation
and reconstructs weights from truncated
activations~\citep{wang2025dobisvddifferentiablesvdllm}. Prehab fine-tunes a
trained model for SVD compression~\citep{qin2025low}.
Appendix~\ref{app:prehab} compares spectral weight decay with
Prehab in this setting. Our main experiments apply it throughout pretraining
and evaluate compression and inference with SVD-LLM.
Appendix~\ref{app:method_choice} compares post-training compression methods.
Section~\ref{sec:experiments} shows the resulting compression and inference gains.


\section{Why Low-Rank Structure Matters}
\label{sec:experiments}
\begin{wrapfigure}{r}{0.45\linewidth}
  \vspace{-\baselineskip}
  \footnotesize
  \begin{algorithm}[width=\linewidth,left=6pt,right=6pt,top=6pt,bottom=6pt,fonttitle=\small]{Adam with spectral WD}
  \label{alg:adam_swd_overview}
  \begin{algorithmic}
    \State $M_0\gets0$, $V_0\gets0$
    \For{$t=1,\ldots,T$}
      \State $G_t\gets\nabla_W\mathcal L(W_{t-1})$
      \State $M_t\gets\beta_1M_{t-1}+(1-\beta_1)G_t$
      \State $V_t\gets\beta_2V_{t-1}+(1-\beta_2)G_t\odot G_t$
      \State $Z_t\gets W_{t-1}-\eta M_t/(\sqrt{V_t}+\varepsilon)$
      \Statex {\setlength{\fboxsep}{2pt}\colorbox{green!18}{\makebox[\dimexpr\linewidth-4pt][l]{$+\quad W_t\gets Z_t-\eta\lambda\,\mathrm{NS}_5(Z_t)$\hfill {\scriptsize $\triangleright\ \mathrm{NS}_5\approx U_ZV_Z^\top$}}}}
      \Statex {\setlength{\fboxsep}{2pt}\colorbox{red!15}{\makebox[\dimexpr\linewidth-4pt][l]{$-\quad W_t\gets Z_t-\eta\lambda W_{t-1}$\hfill {\scriptsize $\triangleright\ \ell_2$ WD $\to$ spectral WD}}}}
    \EndFor
  \end{algorithmic}
  \end{algorithm}
  \vspace{-\baselineskip}
\end{wrapfigure}

Algorithm~\ref{alg:adam_swd_overview} shows how spectral weight decay modifies
Adam by replacing standard pre-step $\ell_2$ decay with the post-step
correction in \eqref{eq:swd_decoupled}. Bias correction is omitted for clarity.
The analogy with $\ell_1$ regularization in linear regression suggests that
this additive spectral shrinkage should promote low-rank structure.
We show that it reduces the effective rank of learned weight matrices and
evaluate the resulting benefits for compression, inference, and robustness
to label noise under fixed-horizon training. Section~\ref{sec:method}
develops the update in detail.

\subsection{Effective rank and why it matters}
\label{sec:erank}

\begin{figure}[t]
  \centering
  \includegraphics[
    width=\textwidth,
    height=0.2281\textwidth
  ]{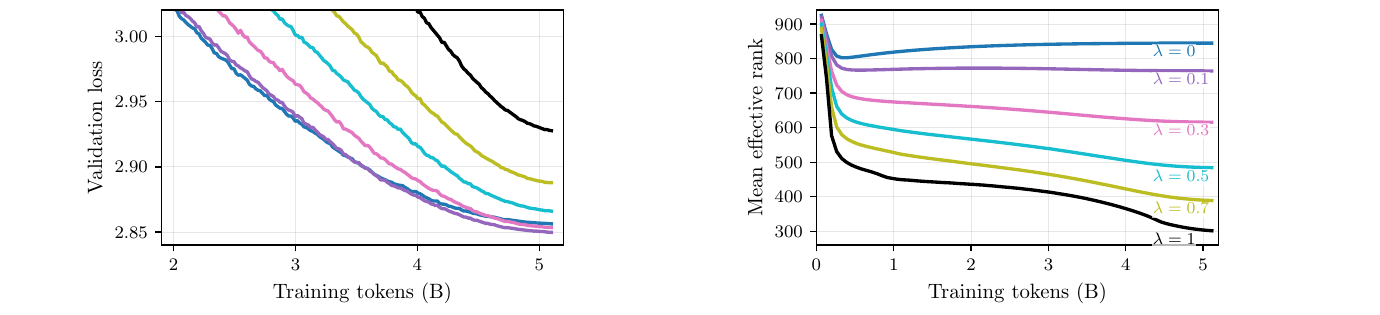}
  \caption{Validation loss (left) and mean effective rank of the weight matrices
  (right) during training of a $257$M-parameter
  LLaMA on FineWeb-Edu with Adam and spectral weight decay.}
  \label{fig:swd_erank}
\end{figure}

We measure spectral concentration with the entropy-based effective
rank~\citep{roy2007effective}. For singular values
$\sigma_1 \ge \dots \ge \sigma_r$ and $p_i = \sigma_i / \sum_j \sigma_j$, it is
\begin{equation}
  \operatorname{erank}(W) = \exp\big(H(p)\big) = \exp\!\Big(-\sum_{i} p_i \log p_i\Big).
  \label{eq:erank}
\end{equation}
Effective rank ranges from $1$ for a rank-one matrix to $r=\min(m,n)$ for a flat
spectrum and decreases as the spectrum becomes more concentrated.
We report a weighted mean across weight matrices, with weights
$\min(m_l,n_l)$, excluding token embeddings and the output head.

Figure~\ref{fig:swd_erank} shows validation loss and mean effective rank during
training of a $257$M-parameter LLaMA~\citep{touvron2023llama} on
FineWeb-Edu~\citep{penedo2024fineweb} with Adam and spectral weight decay.
Appendix~\ref{app:pretraining_setup} summarizes the pretraining configurations.
Larger $\lambda$ lowers effective rank, with a corresponding
validation-loss trade-off. At the end of training, $\lambda=0.7$ reduces mean
effective rank by a factor of $2.17$ relative to $\lambda=0$, with a $1.10\%$
increase in validation loss.

We use effective rank instead of stable rank,
$\|W\|_F^2 / \|W\|_2^2$, because it predicts a practical compression threshold.
Truncating to $\lceil \operatorname{erank}(W) \rceil$ preserves model quality,
while lower ranks degrade it sharply (Figure~\ref{fig:erank_layers}, right).
Stable rank can remain near one when a single singular value dominates many
active directions, so it does not capture this threshold.

\begin{figure}[t]
  \centering
  \includegraphics[width=\textwidth]{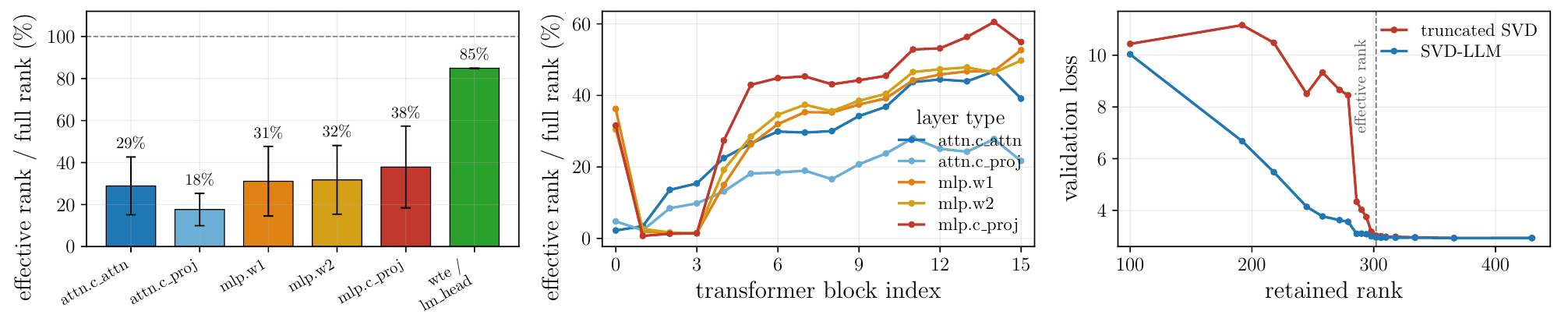}
  \caption{Effective-rank profiles and compression of the $\lambda=1$ model.
  Left: effective rank as a percentage of full
  rank, grouped by layer type. Middle: the same ratio across transformer blocks.
  Right: validation loss after truncated SVD and SVD-LLM at different retained
  ranks. The vertical line marks the mean effective rank.}
  \label{fig:erank_layers}
\end{figure}

\begin{figure}[H]
  \centering
  \includegraphics[width=0.95\textwidth, height=0.45\textwidth]{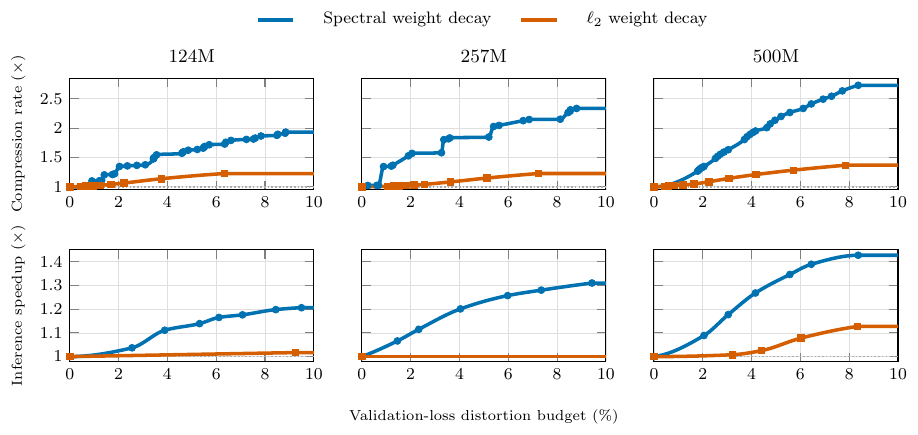}
  \caption{Compression and inference frontiers under validation-loss distortion
  budgets. Top: maximum compression rate relative to the best dense checkpoint
  at each scale. Bottom: maximum inference speedup.}
  \label{fig:compression_inference_tradeoff}
\end{figure}

\subsection{Compression and inference}

Spectral weight decay lowers the effective rank of weight matrices, making them
amenable to low-rank compression. We use SVD-LLM~\citep{wang2024svdllm}, which
gives the lowest validation loss in our matched comparison. We compress each
weight matrix to its own effective rank, which varies substantially across
layers (Figure~\ref{fig:erank_layers}, middle). We leave the token embeddings
and output head uncompressed because their ranks respond little to spectral
weight decay, and compressing them even to their effective ranks sharply
degrades model quality. The comparison and
compression procedure are described in Appendix~\ref{app:method_choice}.
Figure~\ref{fig:compression_inference_tradeoff} (top) reports the maximum
compression rate supported by each validation-loss distortion budget. Spectral
weight decay substantially shifts this frontier upward at all three model
scales. For the $500$M model and a $4\%$ distortion budget, it supports
$1.89\times$ compression, compared with $1.14\times$ after standard weight
decay.

The same factorization reduces the FLOPs of a linear layer when $r(m+n)<mn$.
This reduction does not translate directly into lower latency because the
factorized forward pass evaluates $(U_r\Sigma_r)(V_r^\top X)$ with two GEMMs
instead of the single GEMM $WX$. The additional kernel launch and intermediate
activation can offset the reduction in FLOPs. To improve tensor-core
utilization, we round the effective rank down to the nearest multiple of $16$,
\begin{equation}
  r_l = \max\!\left(16,\,16\left\lfloor
  \frac{\operatorname{erank}(W_l)}{16}\right\rfloor\right),
  \label{eq:aligned_rank}
\end{equation}
Adjusting an aligned rank by one step, $\pm16$, has negligible effect on
validation loss. We therefore use this rank-aligned compression scheme
throughout. The factorized models outperform their dense counterparts only at
large batch sizes, so we measure forward-pass latency with a batch size of $256$
on a single GPU. For the $500$M model, spectral weight decay yields a
$1.09\times$ speedup near a $2\%$ distortion budget and reaches $1.18\times$ at
$4\%$, compared with $1.01\times$ under standard weight decay
(Figure~\ref{fig:compression_inference_tradeoff}, bottom).

\subsection{Robustness to noisy labels}
\label{sec:label_noise_robustness}

Standard $\ell_2$ weight decay limits overfitting~\citep{loshchilov2019decoupled}.
We ask whether the low-rank bias of spectral weight decay provides stronger
protection against unreliable labels, which are common in applied machine
learning~\citep{song2022learning}. We replace $10\%$, $25\%$, $40\%$, or $60\%$
of the training labels uniformly at random and evaluate the final checkpoint
after a fixed training horizon. We assume that no clean validation set is
available, as is often the case in practice, so checkpoint selection cannot
rely on clean-validation early stopping. This protocol exposes memorization
of corrupted labels after the model has learned the clean signal.

\begin{figure}[H]
  \centering
  \includegraphics[width=0.95\textwidth, height=0.5\textwidth]{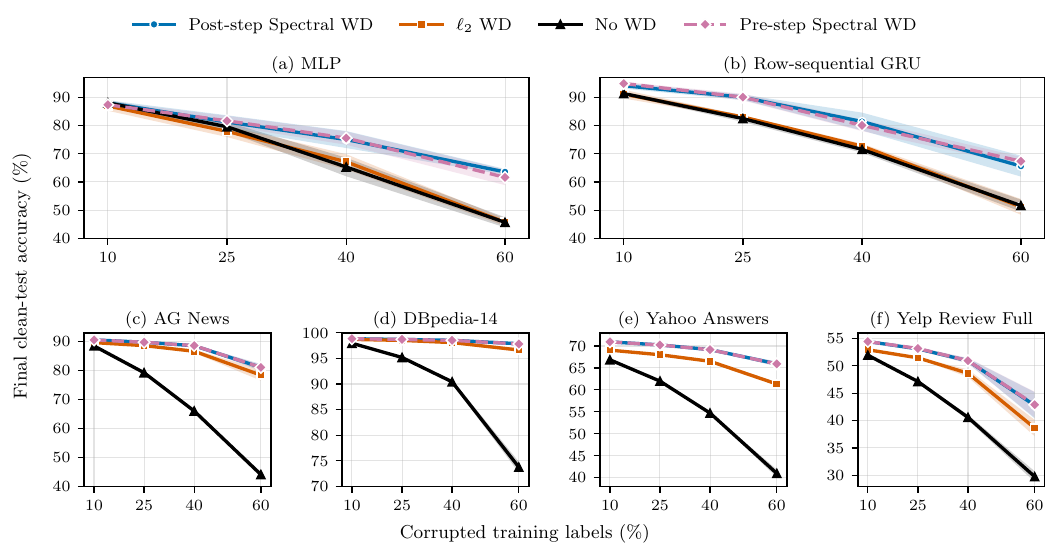}
  \caption{Final-checkpoint clean-test accuracy under label noise for MNIST
  (top) and BERT-base (bottom). Curves show means over five label-corruption
  seeds. Shaded bands show $\pm$ one sample standard deviation.
  Model initialization, data splits, and regularization coefficients are fixed.}
  \label{fig:label_noise_robustness}
\end{figure}

We train a $15$M-parameter MLP and row-sequential
GRU~\citep{cho2014learning} for $60$ epochs on $3{,}000$ MNIST
examples~\citep{lecun1998gradient}. We also fine-tune a $110$M-parameter
BERT-base model~\citep{devlin2019bert} for $25$ epochs on AG News, DBpedia-14,
Yahoo Answers, and Yelp Review Full~\citep{zhang2016character}. On MNIST,
spectral and decoupled $\ell_2$ weight decay act on the same matrices. For BERT,
both methods regularize the displacement $\Delta W=W-W_0$ from pretrained
weights: spectral weight decay penalizes its nuclear norm, while
$\ell_2$-SP~\citep{li2018explicit} penalizes its squared Frobenius norm.
Figure~\ref{fig:label_noise_robustness} shows final-checkpoint clean-test accuracy
as the mean $\pm$ one sample standard deviation over five label-corruption seeds.

On MNIST, the advantage of post-step spectral weight decay
over $\ell_2$ grows with label noise, reaching $17.78$ percentage points at
$60\%$ noise. These gains coincide with less memorization of corrupted labels.
On BERT, post-step spectral weight decay yields higher mean accuracy than $\ell_2$-SP on
every dataset and noise level, with gains of up to $4.59$ points at the highest
noise level.
Pre- and post-step spectral weight decay nearly coincide on BERT, while
neither order is consistently better on MNIST.
Appendix~\ref{app:label_noise_robustness} provides the architectures and
coefficient-selection protocol, together with detailed pre/post-step and baseline
results in Tables~\ref{tab:mnist_mlp_noise_ci}, \ref{tab:mnist_gru_noise_ci},
and~\ref{tab:bert_noise_ci}.

These results suggest a capacity-control effect: by reducing the effective
rank of weights, or their fine-tuning displacements, spectral weight decay
limits the model's ability to memorize corrupted labels.


\section{From Weight Decay to Spectral Shrinkage}
\label{sec:method}

We now examine the update underlying the low-rank behavior observed in
Section~\ref{sec:experiments}. Starting from decoupled weight decay, we derive
spectral shrinkage and contrast it with conventional $\ell_2$ decay. We then
analyze how update order affects the spectral correction and explain our
choice of post-step regularization.

We consider training with a matrix regularizer $h$,
\begin{equation}
  \min_W \big( \mathcal{L}(W) + \lambda h(W) \big),
  \label{eq:reg_obj}
\end{equation}
where $W\in\mathbb{R}^{m\times n}$ is a weight matrix, $\mathcal{L}$ is the
training loss, and $\lambda$ controls the regularization strength. When $h$
is differentiable, a gradient step with learning rate $\eta$ has the form
\begin{equation}
  W^{(k+1)} = W^{(k)}
  - \eta\,\nabla \mathcal{L}\!\left(W^{(k)}\right)
  - \eta\lambda\,\nabla h\!\left(W^{(k)}\right).
  \label{eq:regularized_update}
\end{equation}

AdamW instead decouples regularization from the update on the training
loss~\citep{loshchilov2019decoupled}. Let
$Z^{(k)}=W^{(k)}-\eta\Delta^{(k)}$ be the candidate produced using only
the loss, where $\Delta^{(k)}$ is the optimizer direction. For a differentiable
regularizer, its gradient can be evaluated
before or after the optimizer step:
\begin{equation}
  W_{\mathrm{pre}}^{(k+1)}
  =Z^{(k)}-\eta\lambda\nabla h\!\left(W^{(k)}\right),
  \qquad
  W_{\mathrm{post}}^{(k+1)}
  =Z^{(k)}-\eta\lambda\nabla h\!\left(Z^{(k)}\right).
  \label{eq:decoupled_post}
\end{equation}
Standard AdamW uses the pre-step form for the squared Frobenius penalty, whose
gradient at $W^{(k)}$ is $W^{(k)}$. For spectral weight decay, we use the post-step form. It
regularizes the candidate $Z^{(k)}$, which is the argument of the proximal
operator in proximal-gradient descent (Appendix~\ref{sec:method:wd_prox}).

\subsection{The spectral update}
\label{sec:method:swd}
\vspace{-0.1\baselineskip}

Spectral weight decay uses the nuclear norm
$h(W)=\|W\|_*=\sum_i\sigma_i(W)$ as its matrix penalty. For a compact SVD
$Z^{(k)}=U_Z\operatorname{diag}(\sigma_i)V_Z^\top$, we use the polar factor
$U_ZV_Z^\top$ in place of $\nabla h(Z^{(k)})$ in the post-step rule
\eqref{eq:decoupled_post}. This gives \eqref{eq:swd_decoupled} at iteration $k$.
The polar factor $U_ZV_Z^\top$ is a minimum-Frobenius-norm subgradient of
$\|Z^{(k)}\|_*$~\citep{recht2010guaranteed}.
We approximate $U_ZV_Z^\top$ without an SVD using Newton-Schulz
iterations, using the same five-step polynomial iteration as
Muon~\citep{jordan2024muon}.

\vspace{-0.7\baselineskip}
\paragraph{Comparison with $\ell_2$ weight decay.}
Ordinary weight decay follows from the squared Frobenius penalty
$h_F(W)=\tfrac{1}{2}\|W\|_F^2$, whose gradient is $W$. Fix the current
weights $W=W^{(k)}$ and the loss-only optimizer direction
$\Delta=\Delta^{(k)}$, and set $Z=W-\eta\Delta$ and $\tau=\eta\lambda$.
Let $W=U_W\Sigma_WV_W^\top$ and $Z=U_Z\Sigma_ZV_Z^\top$ be compact SVDs.
The resulting weights are
\begin{equation}
  W_{\ell_2}^{\mathrm{pre}}=Z-\tau W,\qquad
  W_{\ell_2}^{\mathrm{post}}=(1-\tau)Z,\qquad
  W_{\mathrm{spec}}^{\mathrm{pre}}=Z-\tau U_WV_W^\top,\qquad
  W_{\mathrm{spec}}^{\mathrm{post}}=Z-\tau U_ZV_Z^\top.
  \label{eq:wd}
\end{equation}
The pre- and post-step $\ell_2$ updates differ by
$W_{\ell_2}^{\mathrm{pre}}-W_{\ell_2}^{\mathrm{post}}=-\eta^2\lambda\Delta$.
For a bounded optimizer direction, this is $O(\eta^2\lambda)$ for a single
step. For spectral WD, the order
dependence is less straightforward because the task update can change
the polar factor that determines the decay direction. The following result
quantifies the difference.

\begin{proposition}[Sensitivity to update order]
\label{prop:spectral_order}
Let $W,Z=W-\eta\Delta\in\mathbb{R}^{m\times n}$ have full rank, and denote
their smallest singular values by
$\sigma_{\min}(W)$ and $\sigma_{\min}(Z)$. Then
\begin{equation}
  \left\|W_{\mathrm{spec}}^{\mathrm{post}}-
  W_{\mathrm{spec}}^{\mathrm{pre}}\right\|_F
  \leq \frac{2\eta^2\lambda\|\Delta\|_F}
  {\sigma_{\min}(W)+\sigma_{\min}(Z)},
  \qquad
  \left\|W_{\mathrm{spec}}^{\mathrm{post}}-
  W_{\mathrm{spec}}^{\mathrm{pre}}\right\|_2
  \leq 2\eta\lambda.
  \label{eq:spectral_order_bound}
\end{equation}
If both smallest singular values are at least $\gamma>0$ and
$\|\Delta\|_F=O(1)$, the discrepancy is $O(\eta^2\lambda/\gamma)$.
The $2\eta\lambda$ spectral-norm bound is attained near rank deficiency.
\end{proposition}

The Frobenius bound follows from the polar-factor perturbation estimate of
\citet[Proposition~2, Appendix~A.4]{li2026scalingmuondiffusiontransformers}.
Appendix~\ref{app:order_sensitivity}
gives a proof and the quadratic example. Post-step decay
uses the singular directions of the updated weights. Pre-step decay uses
those of $W$, which can mix the current directions when the polar factor
changes. Proposition~\ref{prop:spectral_order} shows that this difference can
be larger than the corresponding pre/post difference under $\ell_2$
weight decay. Appendix~\ref{app:spectral_order_experiments} compares the two
orders empirically.

The two post-step rules then differ in how they shrink the spectrum of
$Z$. $\ell_2$ decay multiplies each singular value by
$1-\tau$, while spectral decay subtracts $\tau$ from each coefficient in
the same singular basis. Above $\tau$, the relative reduction
$\tau/\sigma_i$ is greater for smaller singular values. Below $\tau$, the
spectral step lets the signed coefficient cross zero. The exact nuclear-norm
proximal update instead sets it to zero
(Appendix~\ref{sec:method:wd_prox}). This additive shrinkage should favor
low-rank weights by driving small singular values toward zero, analogous to
the sparsity induced by $\ell_1$ regularization in linear regression
\citep{tibshirani1996regression}. The approximate spectral step need not
produce exact zeros in a single iteration. 

The interpretation of post-step spectral weight decay as an approximate
nuclear-norm proximal step further motivates its use over the pre-step form.
Appendix~\ref{sec:method:wd_prox} develops this interpretation.


\section{Limitations and Future Work}
\label{sec:future_work}
\subsection{Fine-tuning for low-rank compression}
\label{sec:future_work_finetuning}

\begin{wrapfigure}{r}{0.35\linewidth}
  \vspace{-8.7pt}
  \centering
  \includegraphics[width=\linewidth]{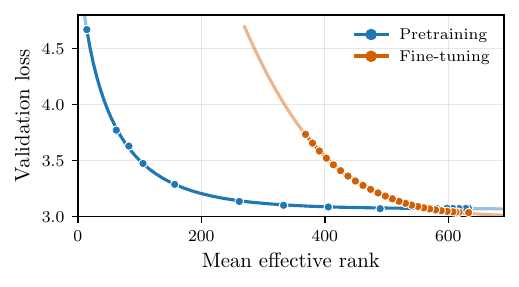}
  \caption{Validation loss versus mean effective rank for $124$M spectral
  weight decay pretraining and $32.8$M-token fine-tuning from an $\ell_2$ WD
  checkpoint.}
  \label{fig:pretrain_finetune_loss_rank}
  \vspace{-\baselineskip}
\end{wrapfigure}

Pretraining is expensive, motivating methods that prepare existing checkpoints
for compression during a short fine-tuning stage. Substantial rank reduction
without loss of model quality remains difficult.

We compare spectral weight decay during pretraining with its use after
pretraining from a standard $\ell_2$ weight decay checkpoint. The $124$M-parameter
model is pretrained for one Chinchilla-optimal training
budget~\citep{hoffmann2022empirical}, then fine-tuned on $32.8$M tokens.
Figure~\ref{fig:pretrain_finetune_loss_rank} plots the final validation loss
against mean effective rank as the spectral coefficient varies. Within the
observed overlap, fine-tuning incurs substantially more loss at the same rank.
Fine-tuning also requires a much larger $\lambda$ before rank changes appreciably.
The sweeps in Figure~\ref{fig:pretrain_finetune_loss_rank} use $\lambda$ values
from $0$ to $7$ for pretraining and from $4$ to $250$ for fine-tuning.
Preparing an existing checkpoint for low-rank compression without a substantial loss increase remains open.
SLORR also reports difficulties in fine-tuning LLMs with low-rank regularization
while preserving their general capabilities~\citep{gonzalezmartinez2026slorr}.
Appendix~\ref{app:prehab} compares spectral weight decay with the fine-tuning
method Prehab~\citep{qin2025low}.

\subsection{Spectral regularization with matrix-aware optimizers}
\label{sec:future_work_optimizers}

The main experiments use Adam-family optimizers. Conventional $\ell_2$ weight
decay is routinely used across optimizer families, so we also examine whether
spectral weight decay transfers beyond Adam. Its decoupled correction can be
applied to the candidate weights produced by any optimizer, but the resulting
loss and rank dynamics depend on the optimizer.

\begin{figure}[H]
  \centering
  \includegraphics[width=\textwidth]{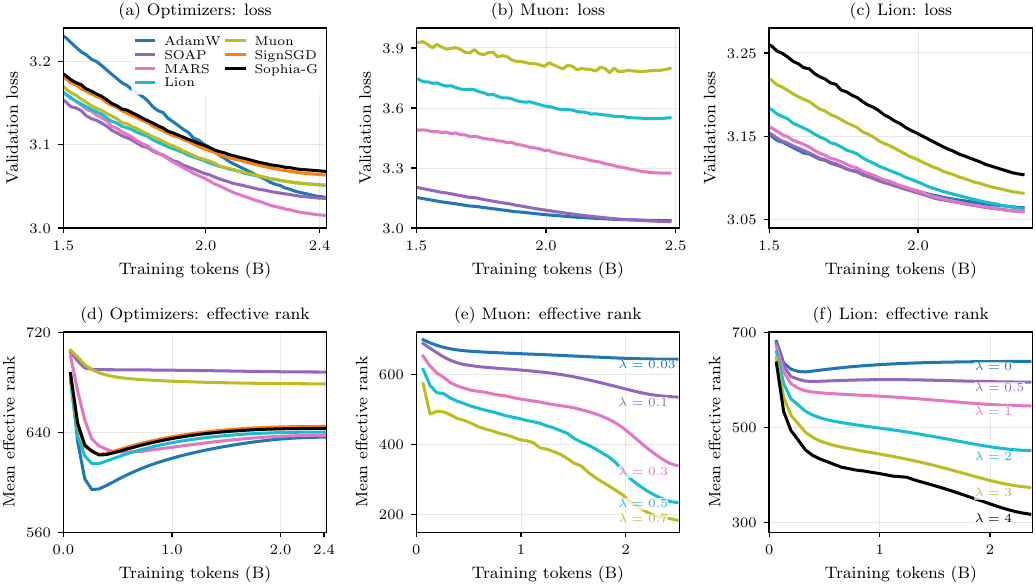}
  \caption{Optimizer-dependent loss and rank dynamics. The top row shows
  validation loss, and the bottom row shows mean effective rank. The columns
  compare $\ell_2$ weight decay with coefficient $0.1$ across optimizers, spectral regularization
  with Muon, and spectral regularization with Lion. Colors match between the
  loss and rank panels.}
  \label{fig:optimizer_l2_rank}
\end{figure}

Figure~\ref{fig:optimizer_l2_rank}(a,d) compares optimizers with the same
$\ell_2$ weight decay coefficient of $0.1$. Muon~\citep{jordan2024muon} and
SOAP~\citep{vyas2024soap} produce higher
effective ranks than the other optimizers, whose final ranks cluster in a
narrow range. This suggests that their optimization dynamics may oppose the
spectral concentration encouraged by spectral weight decay. The Muon sweep
in Figure~\ref{fig:optimizer_l2_rank}(b,e) shows that substantial rank reduction
comes with a steep validation-loss increase, limiting the practical appeal of
this combination. By contrast, Lion exhibits a more favorable
trade-off between loss and rank in Figure~\ref{fig:optimizer_l2_rank}(c,f), qualitatively
similar to Adam. 

The required regularization scale also varies across optimizers. A mean
effective rank of about $300$ requires $\lambda=0.8$ for Adam,
$\lambda=0.4$ for Muon, and $\lambda=4$ for Lion.
Although the $\ell_2$ baselines use a common coefficient, spectral weight
decay requires optimizer-specific calibration for the desired rank and
validation-loss budget.
The strong training results reported for Muon motivate
adapting the techniques studied here to Muon and related matrix-aware
optimizers. Understanding how to combine their optimization benefits with
effective spectral regularization remains a direction for future work.


\section{Conclusion}

We introduced spectral weight decay by replacing the squared Frobenius penalty
of conventional $\ell_2$ weight decay with the nuclear norm. Its decoupled
post-step correction shrinks the spectrum of the updated weights additively,
with a stronger relative effect on smaller singular values. The connection to
approximate proximal descent supports this choice of update order. Our
analysis also shows that pre- and post-step spectral updates can differ more
substantially than their $\ell_2$ counterparts near rank deficiency.

Across $124$M to $500$M LLaMA models trained with Adam-family optimizers, the
resulting low-effective-rank structure improves post-training compression
under matched validation-loss budgets. At $500$M and a $4\%$ distortion budget,
we obtain $1.89\times$ compression and $1.18\times$ GPU inference speedup. The benefits extend beyond compression. Under fixed-horizon
training with noisy labels, spectral regularization reduces memorization and
improves clean-test accuracy on image and text classification tasks. These observations are
consistent with the view that concentrating the weight spectrum limits the
capacity available to fit corrupted labels.

The benefits depend on when and how spectral regularization is applied.
Short fine-tuning from a conventional-weight-decay checkpoint incurs a larger
loss penalty for rank reduction than regularization during pretraining.
Optimizer choice also matters: Adam and Lion exhibit a favorable trade-off between loss and rank,
whereas the Muon runs show a steeper loss increase as rank decreases. Adapting
spectral regularization to matrix-aware optimizers and reducing the cost of
preparing existing checkpoints for low-rank compression remain open problems.


\nocite{kingma2015adam,yuan2025mars,chen2023symbolic,bernstein2018signsgd,liu2024sophia}

\setlength{\bibsep}{8.3pt plus 4pt minus 2pt}
\end{mainpart}

\begin{appendixpart}

\section{Theoretical details}
\label{app:theory}

\subsection{Proof of the spectral update-order bound}
\label{app:order_sensitivity}

For a full-rank matrix $A=U_A\Sigma_AV_A^\top$, write
$Q(A)=U_AV_A^\top$ for its polar factor. We first bound changes in this
factor. The bound holds for rectangular matrices and does not require gaps
between their nonzero singular values.
It is proved by \citet[Proposition~2, Appendix~A.4]{li2026scalingmuondiffusiontransformers}
for full-row-rank matrices. The full-column-rank case follows by transposition.

\begin{lemma}[Polar-factor perturbation]
\label{lem:polar_perturbation}
For full-rank $A,B\in\mathbb{R}^{m\times n}$,
\begin{equation}
  \|Q(A)-Q(B)\|_F
  \leq \frac{2\|A-B\|_F}
  {\sigma_{\min}(A)+\sigma_{\min}(B)}.
  \label{eq:polar_perturbation}
\end{equation}
\end{lemma}

\begin{proof}[Proof of Proposition~\ref{prop:spectral_order}]
By the definitions of the two spectral updates,
\begin{equation}
  W_{\mathrm{spec}}^{\mathrm{post}}-
  W_{\mathrm{spec}}^{\mathrm{pre}}
  =\eta\lambda\bigl(Q(W)-Q(Z)\bigr).
  \label{eq:spectral_order_identity}
\end{equation}
Applying Lemma~\ref{lem:polar_perturbation} with $Z=W-\eta\Delta$ gives
\begin{align*}
  \left\|W_{\mathrm{spec}}^{\mathrm{post}}-
  W_{\mathrm{spec}}^{\mathrm{pre}}\right\|_F
  &\leq \frac{2\eta\lambda\|W-Z\|_F}
  {\sigma_{\min}(W)+\sigma_{\min}(Z)}
  =\frac{2\eta^2\lambda\|\Delta\|_F}
  {\sigma_{\min}(W)+\sigma_{\min}(Z)}.
\end{align*}
Since a full-rank polar factor has spectral norm one, the triangle
inequality also gives
\begin{align*}
  \left\|W_{\mathrm{spec}}^{\mathrm{post}}-
  W_{\mathrm{spec}}^{\mathrm{pre}}\right\|_2
  =\eta\lambda\|Q(W)-Q(Z)\|_2
  \leq\eta\lambda\bigl(\|Q(W)\|_2+\|Q(Z)\|_2\bigr)
  =2\eta\lambda.
\end{align*}

To show that the latter bound is attained near rank deficiency, consider
the fixed loss
$f(X)=\tfrac12\|X-\operatorname{diag}(1,-1)\|_F^2$.
It is $1$-smooth and $1$-strongly convex. For $0<\eta<1$, choose
\begin{equation}
  \begin{gathered}
    W_\eta=\begin{pmatrix}1&0\\0&\eta/2\end{pmatrix},\qquad
    D_\eta=\nabla f(W_\eta)
    =\begin{pmatrix}1&0\\0&\eta/2\end{pmatrix}
    -\begin{pmatrix}1&0\\0&-1\end{pmatrix}
    =\begin{pmatrix}0&0\\0&1+\eta/2\end{pmatrix}.
  \end{gathered}
\end{equation}
The loss-only step with direction $D_\eta$ gives
$Z_\eta=W_\eta-\eta D_\eta
=\operatorname{diag}(1,-\eta(1+\eta)/2)$.
Both matrices have full rank, so their polar factors are unique and
Lemma~\ref{lem:polar_perturbation} applies. Their smallest singular values are
$\eta/2$ and $\eta(1+\eta)/2$, respectively. Thus, they approach rank
deficiency as $\eta\to0$, while the task direction remains bounded:
$\|D_\eta\|_F\leq3/2$. The task update therefore vanishes with $\eta$,
but it changes the sign of the second diagonal entry. For a diagonal matrix
with nonzero entries, the polar factor replaces each diagonal entry by its
sign. Consequently,
$Q(W_\eta)=I$ and $Q(Z_\eta)=\operatorname{diag}(1,-1)$ differ by a
fixed amount even as the weights become arbitrarily close.
Thus, a small task update does not guarantee a small difference between
pre- and post-step spectral weight decay relative to the decay step itself.
Equation~\eqref{eq:spectral_order_identity} therefore gives
\begin{equation*}
  W_{\mathrm{spec}}^{\mathrm{post}}-W_{\mathrm{spec}}^{\mathrm{pre}}
  =\eta\lambda\bigl(Q(W_\eta)-Q(Z_\eta)\bigr)
  =\eta\lambda\begin{pmatrix}0&0\\0&2\end{pmatrix}.
\end{equation*}
This matrix has a single nonzero singular value, $2\eta\lambda$.
Its spectral norm and Frobenius norm therefore coincide:
\begin{equation}
  \left\|W_{\mathrm{spec}}^{\mathrm{post}}-
  W_{\mathrm{spec}}^{\mathrm{pre}}\right\|_2
  =\left\|W_{\mathrm{spec}}^{\mathrm{post}}-
  W_{\mathrm{spec}}^{\mathrm{pre}}\right\|_F
  =2\eta\lambda.
\end{equation}
\end{proof}

\subsection{Proximal interpretation}
\label{sec:method:wd_prox}

Proximal-gradient descent gives a complementary interpretation of the
post-step rule. It first follows the gradient of the training loss, then
regularizes the resulting candidate. For a regularizer $h$ and $\tau>0$, its
proximal operator is~\citep{parikh2014proximal}
\begin{equation}
  \operatorname{prox}_{\tau,h}(Z)
  = \argmin_X \left(h(X)+\frac{1}{2\tau}\|X-Z\|_F^2\right).
  \label{eq:prox_def}
\end{equation}
With $Z=W^{(k)}-\eta\nabla\mathcal{L}(W^{(k)})$ and $\tau=\eta\lambda$,
proximal gradient descent applies
\begin{equation}
  W^{(k+1)} = \operatorname{prox}_{\tau,h}\!\left(
  W^{(k)}-\eta\nabla\mathcal{L}\!\left(W^{(k)}\right)\right).
  \label{eq:prox_gd}
\end{equation}

For the squared Frobenius penalty $h_F(W)=\tfrac{1}{2}\|W\|_F^2$, the exact
proximal map is
$\operatorname{prox}_{\tau,h_F}(Z)=Z/(1+\tau)$. Post-step $\ell_2$ shrinkage,
$(1-\tau)Z$, is its first-order expansion. Pre-step AdamW is also first-order
accurate. \citet{zhuang2022understanding} proved this for vector-valued
parameters, and vectorization gives the matrix statement.

\begin{proposition}
\label{prop:wd_equiv}
Let $h=h_F$, $\tau=\eta\lambda$, and
$Z^{(k)}=W^{(k)}-\eta\nabla\mathcal{L}(W^{(k)})$. For fixed $\lambda$ and
bounded $W^{(k)}$ and $\nabla\mathcal{L}(W^{(k)})$, the pre-step update
$W_{\ell_2}^{\mathrm{pre}}$ in \eqref{eq:wd} differs from the proximal
iteration \eqref{eq:prox_gd} by
$O(\eta^2\lambda)$ as $\eta\to0$.
\end{proposition}

For the nuclear norm, the proximal map instead performs singular-value
soft-thresholding~\citep{cai2010singular}. Given
$Z=U\operatorname{diag}(\sigma_i)V^\top$,
\begin{equation}
  \operatorname{prox}_{\tau,\|\cdot\|_*}(Z)
  =U\operatorname{diag}\!\big(\max(\sigma_i-\tau, 0)\big)V^\top.
  \label{eq:svt}
\end{equation}
Equivalently,
\begin{align}
  \operatorname{prox}_{\tau,\|\cdot\|_*}(Z)
  &=Z-U\operatorname{diag}(\gamma_i)V^\top,
  &
  \gamma_i&=
  \begin{cases}
    \tau, & \sigma_i\geq\tau,\\
    \sigma_i, & \sigma_i<\tau.
  \end{cases}
  \label{eq:swd_exact}
\end{align}
The post-step spectral update \eqref{eq:swd_decoupled} uses $\tau UV^\top$ for
every active singular direction. It matches the proximal correction when
$\sigma_i\geq\tau$. For $0<\sigma_i<\tau$, the exact update instead subtracts
only $\sigma_i$ and sets that direction to zero. With an exact polar factor,
spectral weight decay is thus an approximate nuclear-norm proximal step whose
discrepancy is confined to the sub-threshold tail.

\label{app:svt_compare}

\begin{figure}[t]
  \centering
  \includegraphics[width=\textwidth, height=0.2281\textwidth]{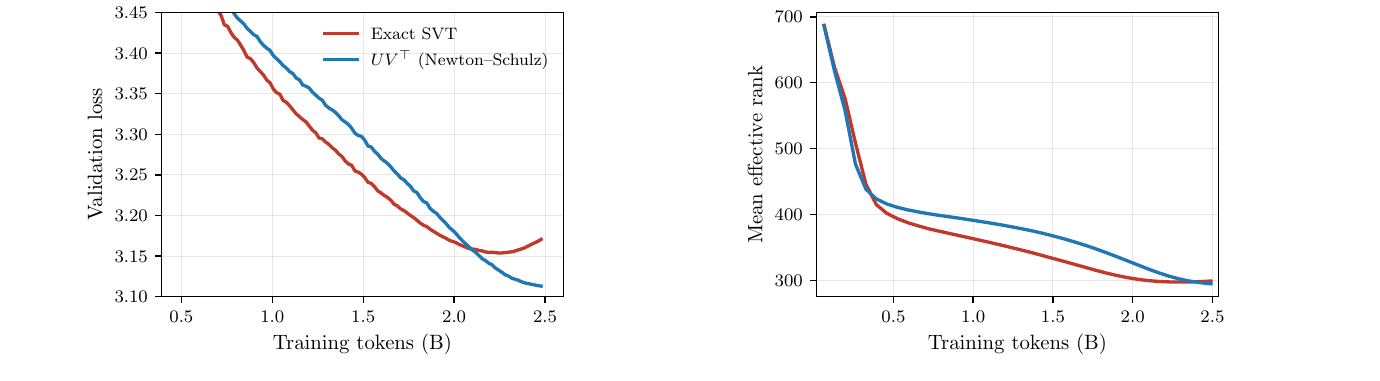}
  \caption{Exact singular-value thresholding and its Newton–Schulz $UV^\top$
  approximation at a matched spectral coefficient on a $124$M-parameter LLaMA
  model. Left: validation loss. Right: mean effective rank of the weight
  matrices. Both updates reach similar effective rank, while the exact update
  ends with higher validation loss.}
  \label{fig:svt_compare}
\end{figure}

To assess how this difference in the sub-threshold tail affects training,
Figure~\ref{fig:svt_compare} compares exact nuclear-norm proximal updates with
post-step spectral weight decay on a $124$M-parameter LLaMA trained on
FineWeb-Edu with $\lambda=1$. In this comparison,
spectral weight decay approximates the polar factor using Newton–Schulz
iterations, while the exact proximal update computes singular-value
thresholding from a full SVD at every step.

Both updates follow nearly identical effective-rank trajectories, but exact
thresholding does not improve validation loss. Its loss initially follows the
Newton–Schulz update, then rises late in training and ends higher. Applying the exact update once
every $s$ steps ($s=10$ to $50$) leaves these results unchanged. We therefore use the less
expensive Newton–Schulz update in the main text.

\subsection{Empirical comparison of pre- and post-step spectral weight decay}
\label{app:spectral_order_experiments}

Proposition~\ref{prop:spectral_order} establishes that the difference between pre-
and post-step spectral weight decay can exceed that of $\ell_2$ weight decay.
We examine the practical consequences of update order through controlled low-rank recovery
problems and language-model pretraining at increasing learning rates.

Both synthetic tasks recover a rank-$16$ matrix $W_\star\in\mathbb R^{128\times128}$
without imposing a rank constraint on the learned matrix. Their convex
data-fitting objectives and known recovery targets allow us to evaluate
update order without the nonconvexity of neural-network training.
We use $1{,}000$ full-batch gradient steps. We compare the exact partial polar
factor with its five-iteration Newton–Schulz approximation ($\mathrm{NS}_5$), both
computed in double precision.
At each noise level, we select $\lambda$ independently for each variant and
implementation on one tuning instance with separate noisy validation data.
We evaluate the final iterate using $e_s=\|W_{T,s}-W_{\star,s}\|_F^2/\|W_{\star,s}\|_F^2$ and plot
the mean error and percentage gain.
We include separately tuned pre- and post-step $\ell_2$ weight decay as a baseline.
Shaded bands show one sample standard deviation of the gain across the five evaluation seeds.

\begin{figure}[H]
  \centering
  \includegraphics[width=\textwidth]{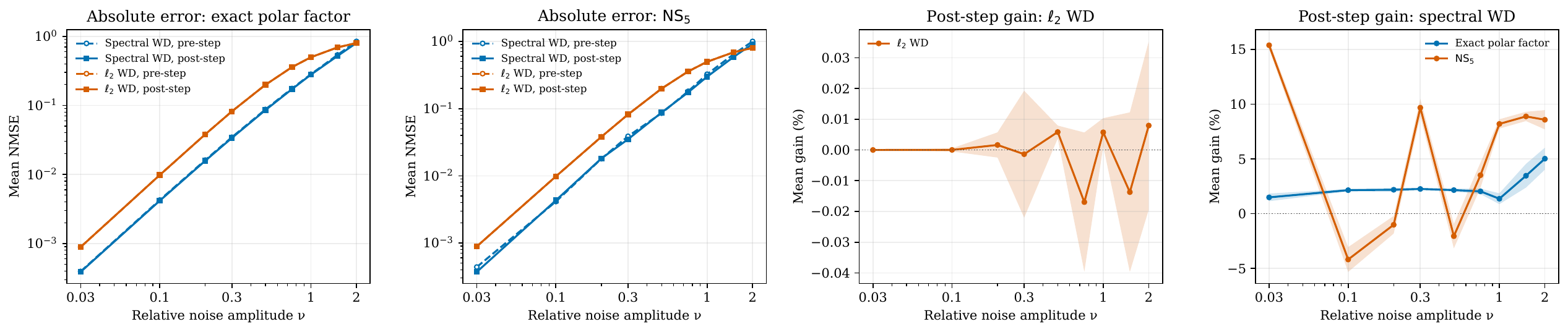}
  \caption{Pre/post-step weight decay in matrix denoising. From left to right:
  mean NMSE with the exact polar factor, mean NMSE with $\mathrm{NS}_5$, post-step gain
  for $\ell_2$ weight decay, and post-step gain for spectral weight decay.
  Absolute errors are averaged over five seeds. Dashed curves denote pre-step
  and solid curves post-step in the first two panels. Gain bands show
  $\pm$ one standard deviation. Positive gains favor post-step.}
  \label{fig:spectral_order_denoising}
\end{figure}

\paragraph{Low-rank recovery by matrix denoising.}
We first study recovery when all matrix entries are observed with noise,
minimizing $\tfrac12\|W-Y\|_F^2$ for
$Y=W_\star+\nu\|W_\star\|_F E/128$, with $E_{ij}\sim\mathcal N(0,1)$.
The nonzero teacher singular values decrease geometrically with a smallest-to-largest
ratio of $0.15$, testing whether regularization preserves weak signal directions
while suppressing noise.
Figure~\ref{fig:spectral_order_denoising} shows that exact post-step decay
reduces mean reconstruction error at all displayed noise levels, with gains
from $1.4\%$ to $5.0\%$. With $\mathrm{NS}_5$, post-step also performs better on average,
although its gains fluctuate more across noise levels. The pre/post-step difference
for $\ell_2$ weight decay remains below $0.02\%$,
negligible compared with the percentage-level gains under spectral weight decay.

\begin{figure}[t]
  \centering
  \includegraphics[width=\textwidth]{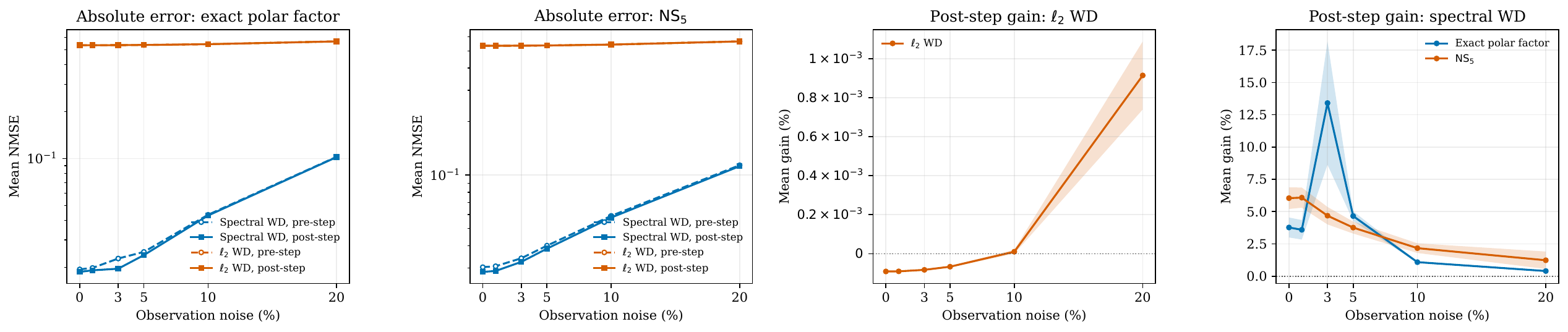}
  \caption{Pre/post-step weight decay in matrix sensing. Panels show mean
  NMSE with the exact polar factor, mean NMSE with $\mathrm{NS}_5$, post-step gain for
  $\ell_2$ weight decay, and post-step gain for spectral weight decay.
  Absolute errors are averaged over five seeds. Dashed curves denote pre-step
  and solid curves post-step in the first two panels. Gain bands show
  $\pm$ one sample standard deviation. Positive gains favor post-step.}
  \label{fig:spectral_order_sensing}
\end{figure}

\paragraph{Low-rank recovery from linear measurements.}
We next replace full noisy observations with $7{,}680$ Gaussian linear
measurements, twice the degrees of freedom
of a rank-$16$ matrix,
$y_i=\langle A_i,W_\star\rangle+\xi_i$, where
$(A_i)_{jk}\sim\mathcal N(0,1)$ and
$\xi_i\sim\mathcal N(0,\nu^2\|W_\star\|_F^2)$.
We minimize the least-squares measurement loss using a flatter teacher spectrum,
with sixteen nonzero singular values ranging linearly from $1$ to $0.6$.
All signal directions are therefore relatively strong, unlike in the denoising task.
Figure~\ref{fig:spectral_order_sensing} shows that post-step achieves lower mean
reconstruction error at every displayed noise level with both the exact polar
factor and $\mathrm{NS}_5$, with $\mathrm{NS}_5$ gains ranging from $1.2\%$ to $6.1\%$.
The $\ell_2$ baseline again differs only by tiny fractions of a percent
between the two update orders.

\paragraph{Update-order sensitivity in language-model pretraining.}
We next examine whether the two variants produce different training dynamics
with Adam on a
$124$M-parameter LLaMA trained on FineWeb-Edu at fixed $\lambda=1$
as the learning rate increases.
Figure~\ref{fig:spectral_order_llm} shows that at $\eta=0.005$, the trajectories
nearly coincide, whereas at $\eta=0.010$ they separate more clearly. Post-step
reaches a lower effective rank but a higher final validation loss at the
largest learning rate. This sweep therefore exposes sensitivity to update
order, but the accompanying change in rank prevents a clear preference
between the variants.

\begin{figure}[H]
  \centering
  \includegraphics[width=\textwidth]{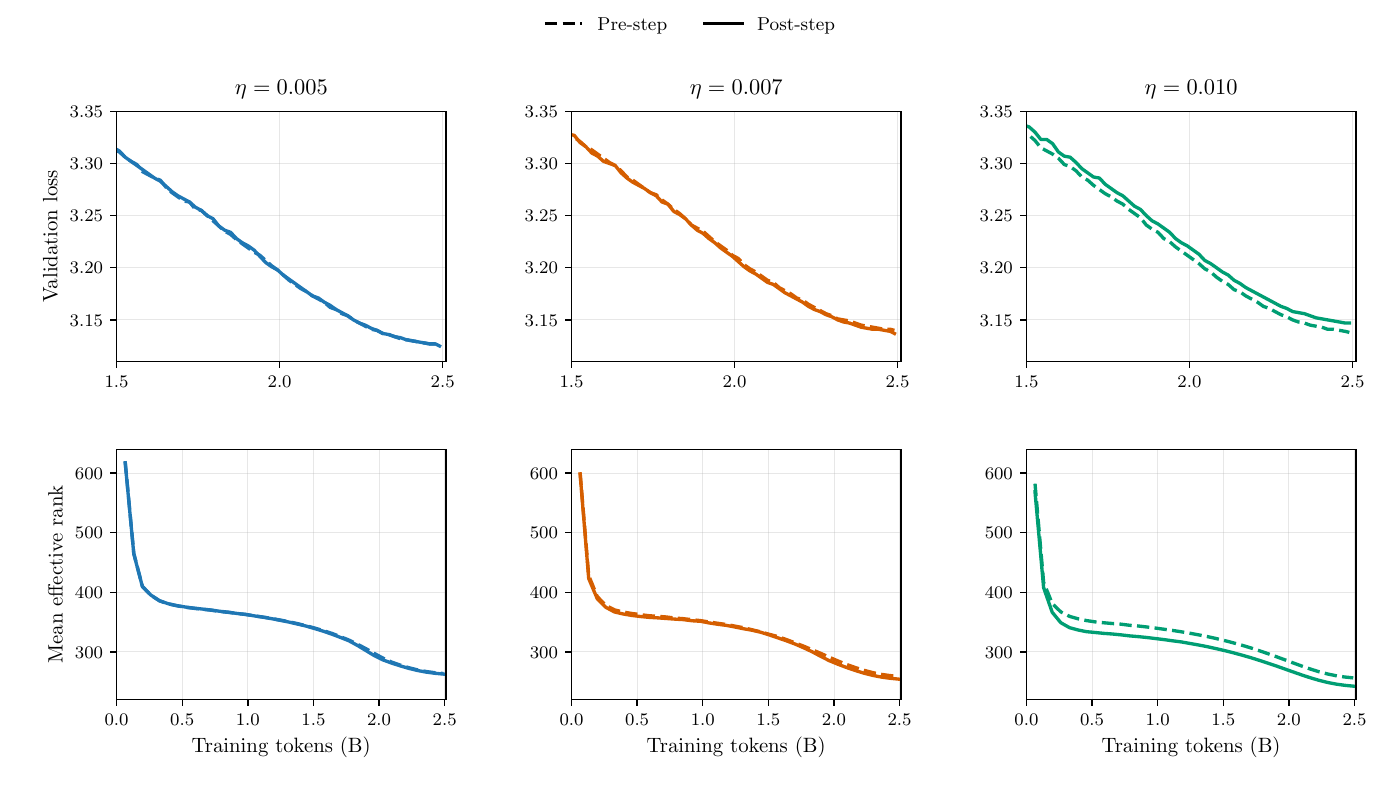}
  \caption{Pre/post-step spectral weight decay with Adam on a $124$M-parameter
  LLaMA trained on FineWeb-Edu at fixed $\lambda=1$. The top row shows validation loss after
  $1.5$B tokens, and the bottom row shows mean effective rank from the start
  of training. Dashed curves denote pre-step and solid curves post-step.}
  \label{fig:spectral_order_llm}
\end{figure}

Our LLM pretraining experiments show little difference between the two update
orders at small learning rates. The synthetic recovery tasks favor post-step
on average, with generally modest gains. Both schemes have the same
computational requirements, so we adopt post-step and recommend it as the default.

\section{Experimental Details and Additional Results}
\label{app:experiments}

\subsection{Language-model pretraining setup}
\label{app:pretraining_setup}

Table~\ref{tab:pretraining_setup} summarizes the main LLaMA pretraining
configurations. We train on FineWeb-Edu using GPT-2 tokenization, rotary
position embeddings, RMSNorm, and SwiGLU feed-forward layers. Input embeddings
and the output head share weights. All runs use sequences of $1{,}024$ tokens,
an effective batch of $128$ sequences, cosine learning-rate decay with
$2{,}000$ warmup steps, Adam betas $(0.9,0.95)$, gradient clipping at $0.5$,
and no dropout. Spectral decay acts on two-dimensional weight matrices using
five Newton-Schulz iterations, as in Muon. The higher-learning-rate comparison
in Appendix~\ref{app:spectral_order_experiments} varies the learning rate while
retaining the $124$M architecture and training horizon.

\begin{table}[ht]
  \centering
  \caption{LLaMA architectures and reference pretraining configurations. Token budgets are rounded.}
  \label{tab:pretraining_setup}
  \small
  \setlength{\tabcolsep}{6pt}
  \begin{tabular}{lrrr}
    \toprule
    Setting & $124$M & $257$M & $500$M \\
    \midrule
    Transformer blocks & $12$ & $16$ & $22$ \\
    Hidden dimension & $768$ & $1{,}024$ & $1{,}280$ \\
    Attention heads & $12$ & $16$ & $20$ \\
    Feed-forward dimension & $2{,}048$ & $2{,}816$ & $3{,}584$ \\
    Embedding vocabulary size & $50{,}304$ & $50{,}304$ & $50{,}304$ \\
    Training steps & $19{,}000$ & $39{,}000$ & $76{,}294$ \\
    Training tokens (B) & $2.49$ & $5.11$ & $10.00$ \\
    \bottomrule
  \end{tabular}
\end{table}

\subsection{Comparison with Prehab}
\label{app:prehab}

Prehab~\citep{qin2025low} adds a smooth rank surrogate over
activation-whitened weights to the task loss. We initialize both methods from
the same $124$M-parameter checkpoint, pretrained with standard $\ell_2$ weight
decay for one Chinchilla-optimal training budget~\citep{hoffmann2022empirical}.
We apply spectral weight decay without modification and tune only $\lambda$.
After $1{,}000$ steps, it reduces mean effective rank by $15.9\%$ with a
$2.46\%$ relative increase in validation loss. At comparable validation loss,
the rank reduction is much larger than with Prehab
(Figure~\ref{fig:swd_prehab}), although the loss cost is still too high for
practical compression.

\begin{figure}[H]
  \centering
  \includegraphics[width=\textwidth, height=0.2281\textwidth]{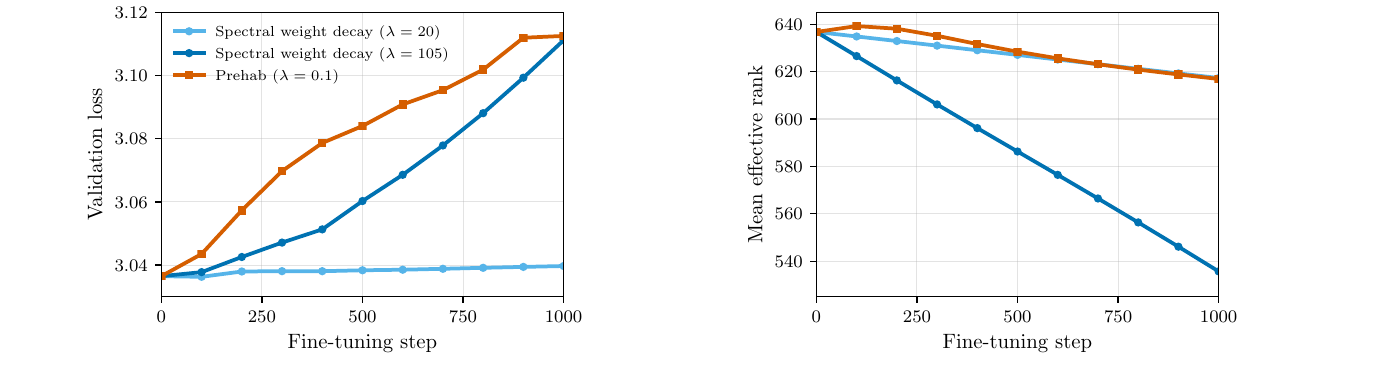}
  \caption{Fine-tuning trajectories from the same $124$M-parameter LLaMA
  checkpoint. At comparable validation loss, spectral weight decay reduces
  effective rank more than Prehab.}
  \label{fig:swd_prehab}
\end{figure}

\subsection{Comparison with Cuttlefish}
\label{app:alternative_methods}

Cuttlefish~\citep{wang2023cuttlefish} is a direct training-time alternative based on a simple idea: it
starts with full-rank training, waits for the stable ranks of the layers to
converge, then factorizes each layer at the corresponding rank and continues
training the low-rank factors. Like spectral weight decay, it modifies the
model during training rather than compressing a fixed checkpoint. Across all
four matched settings in Figure~\ref{fig:swd_cuttlefish}, spectral weight decay
reaches a substantially lower final effective rank at comparable validation
loss.

\begin{figure}[t]
  \centering
  \includegraphics[width=0.9\textwidth, height=0.38\textwidth]{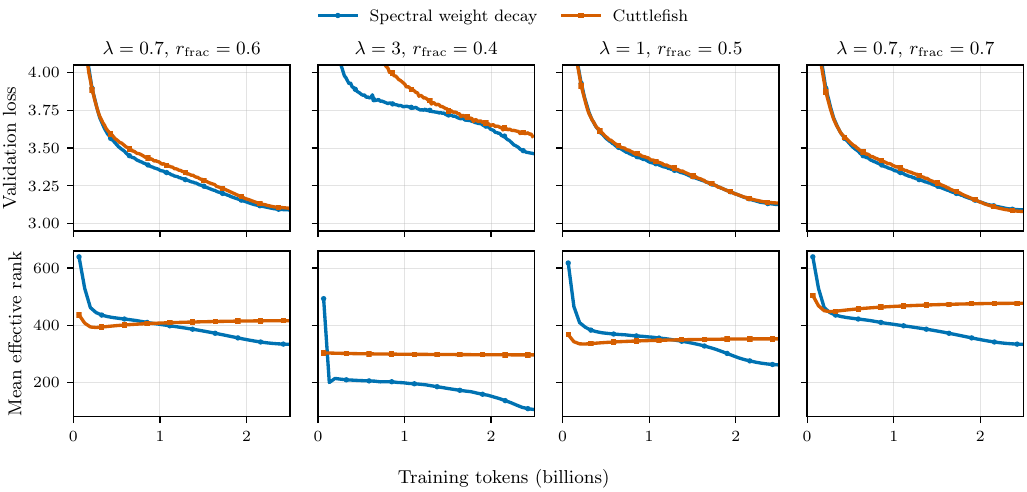}
  \caption{Training trajectories for spectral weight decay and Cuttlefish on
  $124$M-parameter LLaMA models trained on FineWeb-Edu. Each column compares
  the spectral coefficient $\lambda$ with the indicated Cuttlefish rank fraction
  $r_{\mathrm{frac}}$.}
  \label{fig:swd_cuttlefish}
\end{figure}

\subsection{Compression method and procedure}
\label{app:method_choice}

The compression results in Section~\ref{sec:experiments} use SVD-LLM. We compare
five singular-value-based compression methods on one checkpoint to select this
method. Table~\ref{tab:method_choice} applies truncated SVD,
SliceGPT~\citep{ashkboos2024slicegpt}, ASVD~\citep{yuan2023asvd},
SVD-LLM~\citep{wang2024svdllm}, and
Dobi-SVD~\citep{wang2025dobisvddifferentiablesvdllm} to a $124$M-parameter
spectral-weight-decay model with coefficient $1$. We also report accuracy on
ARC-Easy~\citep{clark2018think}, HellaSwag~\citep{zellers2019hellaswag}, and
PIQA~\citep{bisk2019piqa} after compression.

Truncated SVD reaches $1.43\times$ compression with a $+0.0359$ validation-loss
increase. SliceGPT produces a larger loss increase at a lower compression rate.
ASVD and SVD-LLM preserve validation loss better at comparable compression,
with increases of $+0.0286$ and $+0.0196$, respectively. Dobi-SVD reaches a
similar $+0.0207$ increase at $1.21\times$ compression. SVD-LLM gives the lowest
validation loss among the compressed models, so we use it throughout.

\begin{table}[ht]
  \centering
  \caption{Five compression methods applied to the $124$M
  spectral-weight-decay model ($\lambda=1$). Lower losses are better. Higher
  compression rates and downstream accuracies are better. Bold marks the best
  compressed value in each column.}
  \label{tab:method_choice}
  \small
  \setlength{\tabcolsep}{4pt}
  \begin{tabular}{lrrrrrr}
    \toprule
    Method & Compression & val loss & $\Delta$val loss & ARC-E & HellaSwag & PIQA \\
    \midrule
    baseline (uncompressed) & -- & $3.1129$ & -- & $0.4947$ & $0.3004$ & $0.6088$ \\
    \midrule
    truncated SVD & $1.43\times$ & $3.1488$ & $+0.0359$ & $0.4947$ & $0.2999$ & $0.6099$ \\
    SliceGPT      & $1.22\times$ & $4.6347$ & $+1.5218$ & $0.4140$ & $0.2815$ & $0.5680$ \\
    ASVD          & $\mathbf{1.45\times}$ & $3.1416$ & $+0.0286$ & $0.4877$ & $0.2998$ & $\mathbf{0.6121}$ \\
    \textbf{SVD-LLM} & $1.42\times$ & $\mathbf{3.1325}$ & $\mathbf{+0.0196}$ & $\mathbf{0.4982}$ & $0.2991$ & $0.6088$ \\
    Dobi-SVD      & $1.21\times$ & $3.1336$ & $+0.0207$ & $0.4912$ & $\mathbf{0.3005}$ & $0.6110$ \\
    \bottomrule
  \end{tabular}
\end{table}

We leave the token embeddings and output head uncompressed because truncating
them sharply degrades model quality. Spectral weight decay also has little
effect on their ranks (Figure~\ref{fig:erank_layers}, left). We choose a
separate truncation rank for each remaining weight matrix because effective
rank varies substantially across transformer blocks
(Figure~\ref{fig:erank_layers}, middle). After folding the singular values into
one factor, a rank-$r$ approximation of $W\in\mathbb{R}^{m\times n}$ stores
$r(m+n)$ parameters instead of $mn$. Decreasing $r$ therefore saves parameter
memory at the cost of higher validation loss.

\subsection{Noisy-label robustness results}
\label{app:label_noise_robustness}

Tables~\ref{tab:mnist_mlp_noise_ci}, \ref{tab:mnist_gru_noise_ci}, and
\ref{tab:bert_noise_ci} report final-checkpoint clean-test accuracy
as the mean $\pm$ one sample standard deviation, in percentage points,
over five label-corruption seeds.
Model initialization, training order, data splits, and regularization
coefficients are fixed. Standard deviations quantify variability
from label corruption, not from
model initialization or data splitting.

The MNIST MLP flattens each $28\times28$ image and applies four hidden
layers of width $2{,}048$ with ReLU activations, followed by a ten-class
linear head. The row-sequential GRU reads the image as $28$ vectors of
dimension $28$, uses two recurrent layers with hidden dimension $1{,}280$,
and applies a ten-class linear head to the final hidden state of the top
layer.

\begin{table}[ht]
  \centering
  \caption{MNIST MLP: final clean-test accuracy, mean $\pm$ one sample standard deviation.}
  \label{tab:mnist_mlp_noise_ci}
  \small
  \setlength{\tabcolsep}{6pt}
  \begin{tabular}{lrrrr}
    \toprule
    Method & $10\%$ noise & $25\%$ noise & $40\%$ noise & $60\%$ noise \\
    \midrule
    No WD & $87.99 \pm 0.93$ & $79.52 \pm 1.67$ & $65.29 \pm 3.27$ & $45.71 \pm 1.79$ \\
    $\ell_2$ WD & $86.89 \pm 1.70$ & $77.90 \pm 2.00$ & $67.12 \pm 2.71$ & $45.78 \pm 1.29$ \\
    Spectral WD (post-step) & $87.99 \pm 0.93$ & $81.22 \pm 2.36$ & $75.01 \pm 3.01$ & $63.56 \pm 1.04$ \\
    Spectral WD (pre-step) & $87.31 \pm 1.13$ & $81.63 \pm 1.77$ & $75.52 \pm 2.35$ & $61.63 \pm 2.74$ \\
    \bottomrule
  \end{tabular}
\end{table}

\begin{table}[ht]
  \centering
  \caption{MNIST GRU: final clean-test accuracy, mean $\pm$ one sample standard deviation.}
  \label{tab:mnist_gru_noise_ci}
  \small
  \setlength{\tabcolsep}{6pt}
  \begin{tabular}{lrrrr}
    \toprule
    Method & $10\%$ noise & $25\%$ noise & $40\%$ noise & $60\%$ noise \\
    \midrule
    No WD & $91.29 \pm 0.64$ & $82.42 \pm 1.18$ & $71.50 \pm 1.13$ & $51.65 \pm 2.12$ \\
    $\ell_2$ WD & $91.07 \pm 1.55$ & $82.98 \pm 0.54$ & $72.69 \pm 0.84$ & $51.20 \pm 2.78$ \\
    Spectral WD (post-step) & $94.06 \pm 1.03$ & $89.96 \pm 0.97$ & $81.38 \pm 3.15$ & $65.60 \pm 3.70$ \\
    Spectral WD (pre-step) & $94.83 \pm 0.62$ & $90.00 \pm 1.17$ & $80.02 \pm 2.03$ & $67.33 \pm 1.52$ \\
    \bottomrule
  \end{tabular}
\end{table}

\phantomsection
\label{app:bert_robustness_architecture}
We use BERT-base-uncased~\citep{devlin2019bert}, initialized from the
\texttt{google-bert/bert-base-uncased} checkpoint. The encoder has $12$
transformer blocks, each with $12$ attention heads, hidden dimension $768$,
and feed-forward dimension $3072$ with GELU activations. Its WordPiece
vocabulary contains $30{,}522$ tokens. The checkpoint supports $512$ positions,
but our inputs are truncated to $128$ tokens. Attention and hidden-layer
dropout probabilities are both $0.1$. A linear classifier maps the pooled
\texttt{[CLS]} representation to $4$, $14$, $10$, or $5$ classes for AG News,
DBpedia-14, Yahoo Answers, and Yelp Review Full, respectively.

The model has approximately $110$M parameters, of which $29$M are trainable.
We freeze the embeddings and encoder blocks $1$ through $8$, and train blocks
$9$ through $12$, the pooler, and the classifier. Both regularizers act on the
displacements from the pretrained checkpoint of the attention and feed-forward
weight matrices in the trainable blocks and the pooler weight matrix.
The classifier, biases, and LayerNorm parameters are not regularized.

Regularization coefficients are calibrated in separate tuning runs using
final clean-validation accuracy, independently at each noise level for MNIST
and at $60\%$ noise for BERT. The selected BERT coefficients are then fixed
across noise levels. All coefficients are frozen before the five new
label-corruption replications. Clean validation is not used to select
checkpoints, and the replication test results are not used to tune
coefficients. Each reported checkpoint is taken at the fixed training horizon.

\begin{table}[ht]
  \centering
  \caption{BERT-base: final clean-test accuracy, mean $\pm$ one sample standard deviation.}
  \label{tab:bert_noise_ci}
  \small
  \setlength{\tabcolsep}{5pt}
  \begin{tabular}{llrrrr}
    \toprule
    Dataset & Noise & No WD & $\ell_2$-SP & \shortstack{Spectral WD\\(post-step)} & \shortstack{Spectral WD\\(pre-step)} \\
    \midrule
    AG News & $10\%$ & $88.49 \pm 0.10$ & $89.63 \pm 0.03$ & $90.57 \pm 0.10$ & $90.55 \pm 0.11$ \\
     & $25\%$ & $79.23 \pm 0.47$ & $88.60 \pm 0.21$ & $89.77 \pm 0.11$ & $89.76 \pm 0.13$ \\
     & $40\%$ & $66.02 \pm 0.50$ & $86.65 \pm 0.55$ & $88.56 \pm 0.33$ & $88.61 \pm 0.31$ \\
     & $60\%$ & $44.05 \pm 0.85$ & $78.40 \pm 1.91$ & $81.07 \pm 1.14$ & $81.08 \pm 1.10$ \\
    \midrule
    DBpedia-14 & $10\%$ & $97.97 \pm 0.11$ & $98.72 \pm 0.01$ & $98.86 \pm 0.03$ & $98.85 \pm 0.04$ \\
     & $25\%$ & $95.19 \pm 0.18$ & $98.52 \pm 0.13$ & $98.72 \pm 0.10$ & $98.72 \pm 0.10$ \\
     & $40\%$ & $90.43 \pm 0.38$ & $98.10 \pm 0.25$ & $98.55 \pm 0.21$ & $98.53 \pm 0.20$ \\
     & $60\%$ & $73.74 \pm 1.05$ & $96.65 \pm 0.36$ & $97.88 \pm 0.34$ & $97.83 \pm 0.30$ \\
    \midrule
    Yahoo Answers & $10\%$ & $66.86 \pm 0.23$ & $69.06 \pm 0.14$ & $70.94 \pm 0.09$ & $70.95 \pm 0.12$ \\
     & $25\%$ & $61.99 \pm 0.17$ & $68.01 \pm 0.12$ & $70.25 \pm 0.14$ & $70.24 \pm 0.13$ \\
     & $40\%$ & $54.71 \pm 0.15$ & $66.52 \pm 0.33$ & $69.18 \pm 0.22$ & $69.17 \pm 0.19$ \\
     & $60\%$ & $40.95 \pm 0.88$ & $61.34 \pm 0.37$ & $65.93 \pm 0.53$ & $65.92 \pm 0.52$ \\
    \midrule
    Yelp Review Full & $10\%$ & $51.93 \pm 0.10$ & $52.92 \pm 0.14$ & $54.42 \pm 0.12$ & $54.41 \pm 0.15$ \\
     & $25\%$ & $47.12 \pm 0.16$ & $51.44 \pm 0.11$ & $53.08 \pm 0.20$ & $53.15 \pm 0.25$ \\
     & $40\%$ & $40.58 \pm 0.48$ & $48.62 \pm 0.70$ & $50.91 \pm 0.46$ & $50.92 \pm 0.51$ \\
     & $60\%$ & $29.76 \pm 0.97$ & $38.62 \pm 1.50$ & $42.78 \pm 2.39$ & $42.89 \pm 2.37$ \\
    \bottomrule
  \end{tabular}
\end{table}


\end{appendixpart}
\end{document}